\documentclass[conference]{IEEEtran}
\IEEEoverridecommandlockouts
\usepackage{cite}
\usepackage{amsmath,amssymb,amsfonts}
\usepackage{amsthm}
\usepackage{algorithm}
\usepackage{algorithmic}
\usepackage{graphicx}
\usepackage{booktabs}
\usepackage{textcomp}
\usepackage{xcolor}
\def\BibTeX{{\rm B\kern-.05em{\sc i\kern-.025em b}\kern-.08em
    T\kern-.1667em\lower.7ex\hbox{E}\kern-.125emX}}
\def\*#1{\mathbf{#1}}
\def\~#1{\boldsymbol{#1}}

\newtheorem{theorem}{Theorem}
\newtheorem{definition}{Definition}

\begin{document}

\title{Generalized Groves of Neural Additive Models: Pursuing Transparent Machine Learning Models in Finance
}

\author{\IEEEauthorblockN{Dangxing Chen}
\IEEEauthorblockA{\textit{Zu Chongzhi Center for Mathematics and Computational Sciences} \\
\textit{Duke Kunshan University}\\
Kunshan, China \\
dangxing.chen@dukekunshan.edu.cn}
\and
\IEEEauthorblockN{Weicheng Ye}
\IEEEauthorblockA{
New York, US \\
}
}

\author{\IEEEauthorblockN{
Dangxing Chen  \IEEEauthorrefmark{1} \thanks{$^*$ Corresponding author.} and 
Weicheng Ye
}
\IEEEauthorblockA{Zu Chongzhi Center for Mathematics and Computational Sciences\\
Duke Kunshan University, Kunshan, Jiangsu, China\\
Email: dangxing.chen@dukekunshan.edu.cn and 
weicheng.ye@dukekunshan.edu.cn}
}

\maketitle

\begin{abstract}

While machine learning methods have significantly improved model performance over traditional methods, their black-box structure makes it difficult for researchers to interpret results. For highly regulated financial industries, model transparency is equally important to accuracy. Without understanding how models work, even highly accurate machine learning methods are unlikely to be accepted. We address this issue by introducing a novel class of transparent machine learning models known as generalized groves of neural additive models. The generalized groves of neural additive models separate features into three categories: linear features, individual nonlinear features, and interacted nonlinear features.
Additionally, interactions in the last category are only local. A stepwise selection algorithm distinguishes the linear and nonlinear components, and interacted groups are carefully verified by applying additive separation criteria. Through some empirical examples in finance, we demonstrate that generalized grove of neural additive models exhibit high accuracy and transparency with predominantly linear terms and only sparse nonlinear ones.

\end{abstract}


\section{Introduction}

Machine learning (ML) models have been proven extremely successful in analyzing complex and high-dimensional datasets, with improved accuracy over traditional methods, such as linear and logistic regressions (LaLRs). 
On the other hand, there has been an increase in public concern about the use of ML methods without enhanced regulation. As of April 2021, the European Commission (EC) has proposed the Artificial Intelligence Act (AIA) \cite{EUact}, which marks a historic first step towards filling the regulatory gap. Additionally, the review article \cite{carlo2021AI} explains why regulators are obliged to require ML methods to be transparent and explainable. 

In the highly regulated financial sector, transparency and explainability are equally important to the model accuracy. In the handbook on model risk management of the US Office of the Comptroller of the Currency (OCC) published in August 2021, it stressed the importance of evaluation transparency and explainability for risk management when using complex models\cite{OCC2021model}.  
More recently in May 2022, the Consumer Financial Protection Bureau (CFPB) confirmed that anti-discrimination laws require institutions to provide a detailed explanation to consumers when denying a credit application using ML methods \cite{CFPB2022credit}. Researchers are investigating explainable ML tools in light of the growing regulatory requirements  \cite{yang2020enhancing,rudin2019stop,wang2017bayesian}.

Specifically, two directions have been extensively explored by researchers in order to provide explainability. In the first direction, model-agnostic approaches are provided to disentangle a trained black-box model. Several popular methods have been developed including locally interpretable model-agnostic explanations (LIME) \cite{ribeiro2016should}, SHapley Additive Explanations (SHAP) \cite{lundberg2017unified}, and sensitivity-based analysis \cite{horel2020significance}. Despite these successes, it is important to note that ML methods may be intrinsically opaque, as opposed to LaLRs. Therefore, while such explainability may meet explanation requirements for applications in fields of text and image analysis, they may not be adequate enough in financial applications. Furthermore, universal explanations do not exist, and there has been criticism of blindly adopting them \cite{molnar2020pitfalls,kumar2020problems,rudin2019stop,slack2020fooling}. In the second approach, it simplifies model architecture by enhancing its transparency, see \cite{yang2020enhancing, yang2021gami, chen2018interpretable, agarwal2021neural, dumitrescu2022machine}. Explainability and transparency are closely connected concepts: a transparent model is generally easy to comprehend. A particular focus is given to \cite{agarwal2021neural} for a novel class of Neural Additive Models (NAMs). This type of model integrates the approximation capabilities of neural networks (NNs) with the interpretability of general additive models (GAMs) by using a linear combination of NNs. Having been inspired by that, we hope to provide a highly accurate ML method that provides the greatest transparency. 

With transparency in mind, we ask the following question in this paper: given a dataset, what is the simplest architecture of additive models that provides high accuracy? Informally, we are looking for models to satisfy the following conditions:
\begin{itemize}
    \item Based on the data, the model accuracy is comparable with the most accurate ML model.
    \item If the relationship between the feature and the output is linear, we should maintain the output with respect to the feature in its linear additive form.
    \item There is a maximum number of disjoint subsets that cover all feature indices $(1,\dots,p)$, such that each feature can interact only with its group. In other words, if there are no direct or indirect interactions between two features, the output with respect to them should be separated into an additive form. 
\end{itemize}
We wish to adopt the most transparent model if accuracy is not harmed. As pointed out in the model risk management handbook by OCC  \cite{OCC2021model}, ``model risks increase with greater model complexity." Consequently, transparency is a key factor in the selection of models by the institutions. 

The recently introduced NAMs \cite{agarwal2021neural} have provided a very simple architecture, but it does not explicitly incorporate the linear component and it does not yet include interactions. As a result, it does not necessarily have the simplest architecture and might lose accuracy when interactions exist. We propose a forward stepwise selection algorithm to include linear components. 
To simplify nonlinear structures further, we closely examine the additive separability between two features using the additive separability theorem and the universal approximation property of NNs, motivated by \cite{sorokina2008detecting,tsang2018neural,tsang2020does}. 

After incorporating all ingredients, we present a novel class of generalized groves of neural additive models (GGNAMs). GGNAMs categorize all features into three categories: (1) linear features; (2) individual nonlinear features; and (3) nonlinear features that interact. In the case of interacted features, they only interacted with features within the same subset. As a result, interactions occurred locally. 
GGNAMs improve the simplest LaLR with great accuracy while requiring minimal modifications and provide a smooth transition from commonly used traditional methods to advanced ML methods. Some of the advantages of GGNAMs are as follows:
\begin{itemize}
    \item Transparent, and therefore easy to explain.
    \item Friendly to the evaluation of conceptual soundness, detailed post-analyses can be provided. 
    \item Build a bridge between traditional LaLR and state-of-the-art ML, with minimal modifications required.
    \item Model performance compares favorably with complex black-box ML methods.
\end{itemize}


\subsection{Relationship to Existing Literature}

There has been an increasing trend toward exploring transparent models \cite{rudin2019stop}. In general, there are two types of existing approaches: (1) building predetermined transparent models \cite{agarwal2021neural,yang2021gami,chen2022monotonic,chen2023address,lou2013accurate}; (2) simplifying complex black-box ML models \cite{chen2016infogan,tsang2018neural,sorokina2008detecting}. The first approach is easier to understand and implement with very robust results. However, it might be less accurate due to the absence of complex interactions. Our method falls into the second group and allows complex interactions. We differ from other methods in two ways: (1) for architectures, we include linear components to enhance the transparency of the model architecture, and empirical results suggest that linear relationships are sufficient for a large number of features in many datasets; (2) for computations, we examine additive separability, similar to \cite{sorokina2008detecting}, but provide rigorous theoretical support. In comparison with the regularization approach \cite{tsang2018neural,yang2021gami}, this approach provides a more robust and transparent architecture, but it also requires more intensive computational resources. By using a forward stepwise selection of linear components, we reduce the computation.

\section{Prerequisites}

In this section, we briefly review neural additive models (NAMs). Assume we have $\mathcal{D} \times \mathcal{Y}$, where $\mathcal{D}$ is the dataset with $n$ samples and $p$ features and $\mathcal{Y}$ is the corresponding numerical values in regression and labels in classification. We assume the data-generating process (DGP) of
\begin{align}
y = f(\*x) + \epsilon
\end{align}
for regression problems and 
\begin{align}
 y|\*x = \text{Bernoulli}(f(\*x))    
\end{align}
for binary classification problems. Then machine learning (ML) methods are applied to approximate $f$. 

Fully-connected neural networks (FCNNs) have been very successful for approximating high-dimensional complex functions, due to their universal approximation property. Despite their success in approximation, their complicated deep layers with massively connected neurons prevent us from interpreting the result. Neural additive models (NAMs) \cite{agarwal2021neural} improve the explainability of FCNNs by restricting the architecture of neural networks (NNs). NAMs belong to the family of generalized additive models (GAMs) of the form
\begin{align} \label{eq:GAM}
g(\mathbb{E}[y|\*x]) = \alpha + f_1(x_1) + \dots + f_p(x_p),
\end{align}
where $\*x = (x_1, \dots, x_p)$ is the input with $p$ features, $y$ is the target variable, $g(\cdot)$ is the link function (e.g., logistic link function in classification). For NAMs, each $f_i$ is parametrized by a neural network. In an example of four features with one hidden layer, the architecture of a NAM in Figure~\ref{fig:NAM} is compared with a FCNN in Figure~\ref{fig:ANN}.  NAMs are capable of learning arbitrary complex functions if there are no interactions of $f_i$. There are several advantages of NAMs, including approximation capability, transparency, and explainability. Interested readers are referred to the summary in \cite{agarwal2021neural}.  NAMs, however, do not account for the potential complex interactions. While low-order interactions can be added to enhance model complexity \cite{yang2021gami}, such architectures are still somewhat predetermined and may be inadequate for complex datasets. Additionally, it is possible to overlook potential linear structures in $f_i$.

\begin{figure}
  \centering
  \includegraphics[scale=0.4]{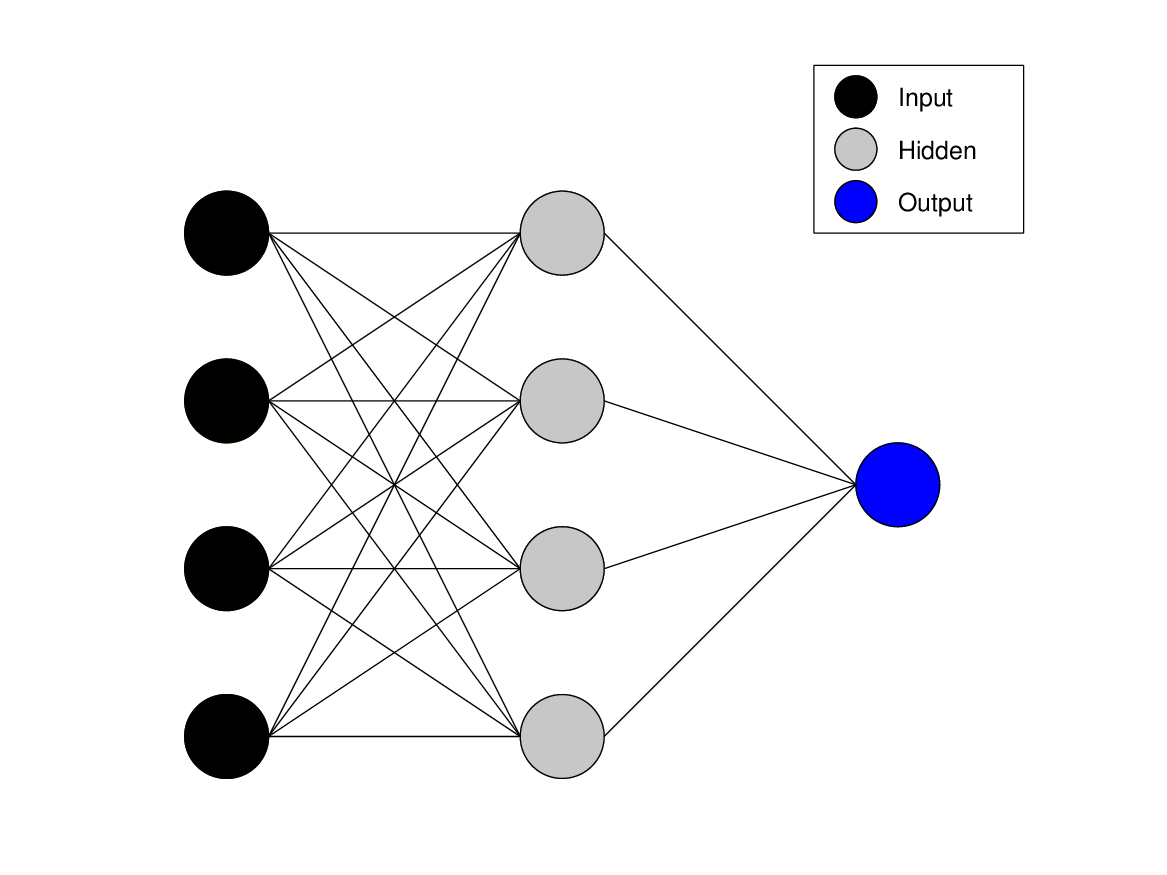}
  \caption{An architecture of a FCNN.}
  \label{fig:ANN}
\end{figure}

\begin{figure}
  \centering
  \includegraphics[scale=0.4]{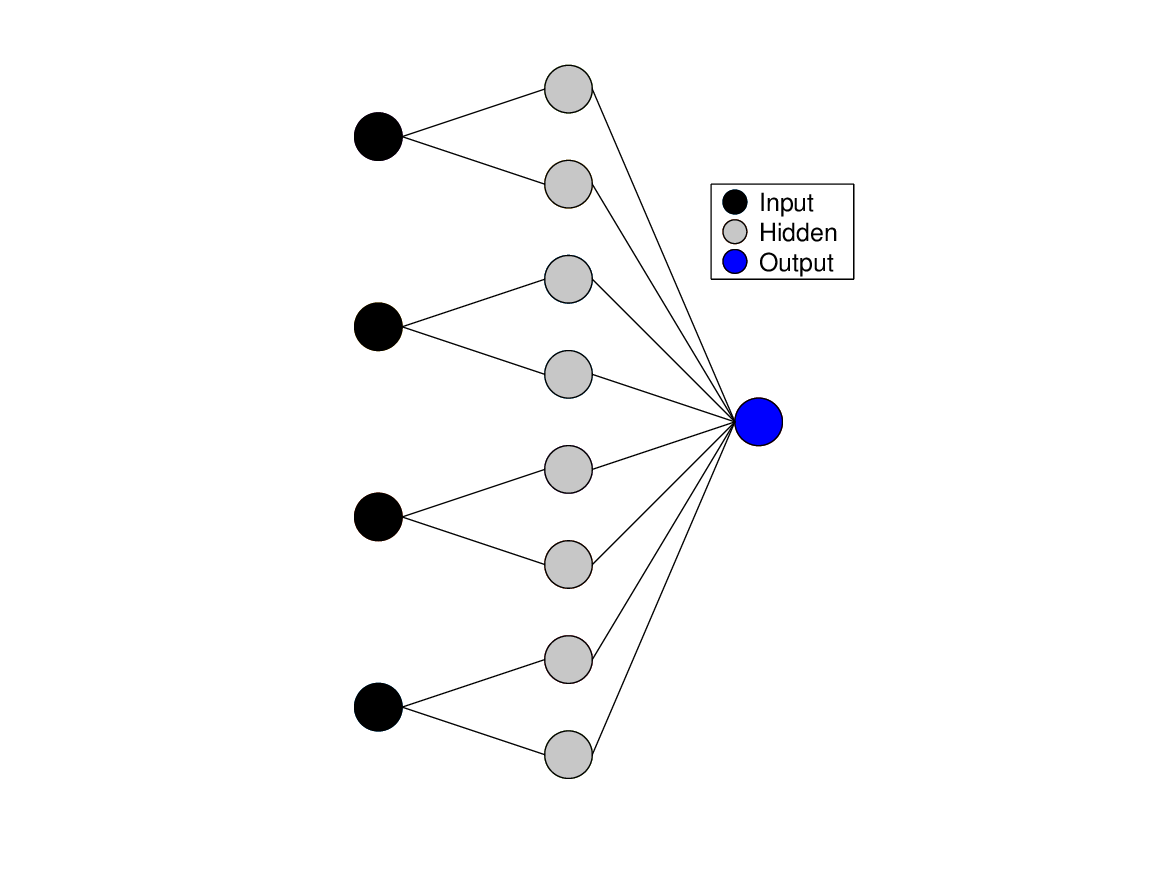}
  \caption{An architecture for a NAM. Features do not interact. }
  \label{fig:NAM}
\end{figure}

\section{Generalized groves of neural additive models}

We intend to develop machine learning models based on the additive form which has the simplest architecture with high accuracy. We will first separate linear and nonlinear components using the forward stepwise selection algorithm. We then propose generalized groves of neural additive models to further sparsify nonlinear structures with the method to identify statistical interactions.

\subsection{Forward stepwise selection}

We now show how to separate linear and nonlinear components. Suppose the input $\*x$ with the feature index  $D = \{1, \dots, p\}$ can be split into linear components $\*x_U$ and nonlinear components $\*x_V$, where $U \cup V = D$ and $U \cap V = \emptyset$, then we assume $f$ takes the additive form of
\begin{align} \label{eq:GNML}
    g(\mathbb{E}[y|\*x]) = \alpha + \sum_{u: u \in U} \beta_{u} x_{u} + f_{V}(\*x_{V}),
\end{align}
where $f_{V}$ is parametrized by a FCNN. In the case of $U=\emptyset$, $g$ reduces to a FCNN; in the case of $V=\emptyset$, $g$ reduces to a  LaLR. An example of the architecture with two linear features and two nonlinear features is shown in Figure~\ref{fig:GNML}. The new architecture \eqref{eq:GNML} improve NAMs by including linear components, and allowing complex interactions between nonlinear components, but at the expense of neglecting potential sparse additive forms in nonlinear components. This will be explored later on. 

Consequently, we need an algorithm to distinguish between linear and nonlinear components. In our empirical experiments, we have found that in many datasets, only a few features exhibit nonlinear behaviors, suggesting that we could use a forward stepwise selection method to distinguish linear from nonlinear features. The procedure is summarized in Algorithm~\ref{alg:forward_selection}. A backward stepwise selection is certainly another option and will be similar. The worst case scenario of the forward selection algorithm requires $\frac{p(p-1)}{2}$ times of training. But, as we will see in empirical experiments, this could happen early. For problems with large $p$, it may be applied after feature selection. As the focus of this paper is not on general feature selection, we will not discuss it further. 



The forward selection algorithm provides an efficient method of distinguishing linear components from nonlinear ones. This method is helpful for reducing the dimensionality of nonlinear features and serves as an initialization method for more complicated models, as we will discuss below. We should clarify that \eqref{eq:GNML} is a stage on the way from NAMs to GGNAMs, in order to demonstrate the process.

\begin{figure}
    \centering
    \includegraphics[scale=0.4]{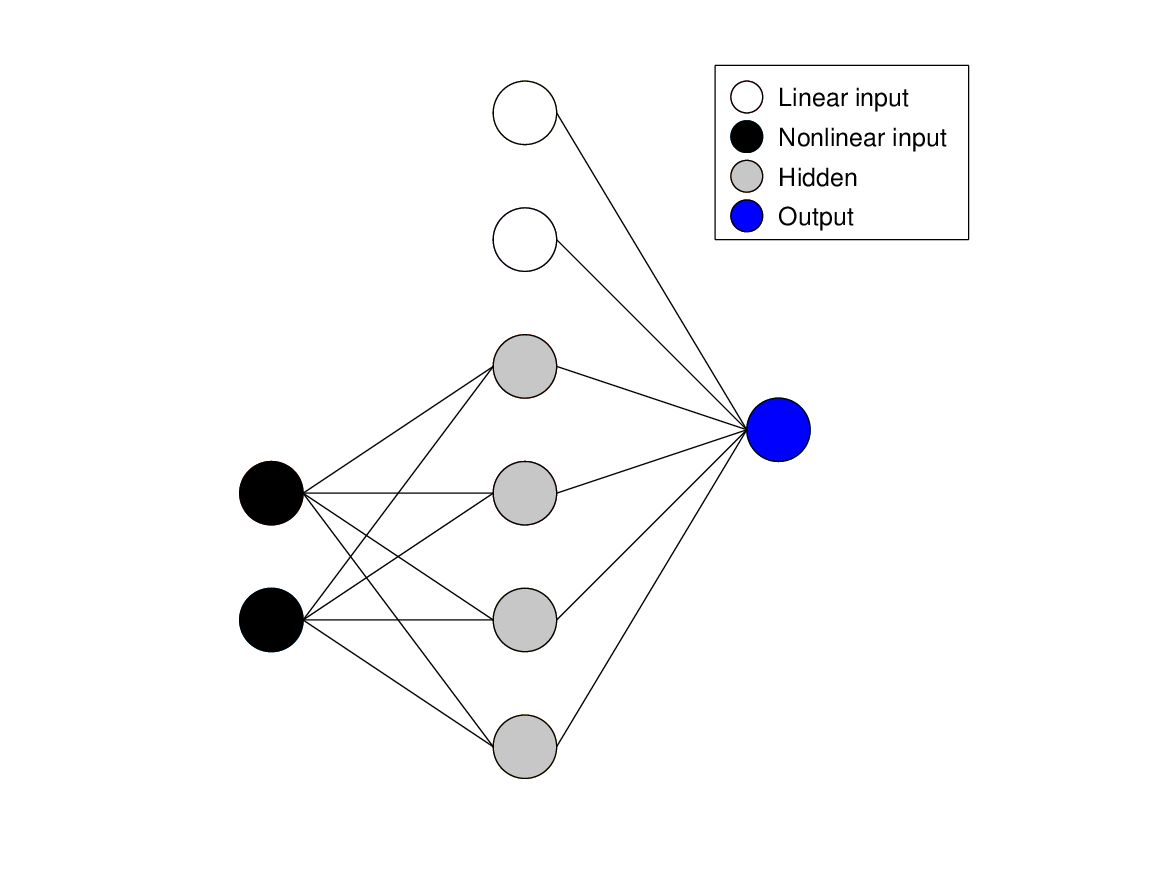}
    \caption{An architecture of \eqref{eq:GNML}. Nonlinear features have complex interactions. }
    \label{fig:GNML}
\end{figure}

\begin{algorithm} [ht!]
\caption{Forward stepwise selection} 
\label{alg:forward_selection}
\begin{algorithmic}[1]
\STATE \textbf{Inputs}: $\epsilon>0$
\STATE Train $f^0$ with loss $e^0$  by a LaLR
\STATE Train $f^p$ with loss $e^p$ by a FCNN
\STATE $i=1$
\STATE $U = (1, \dots, p)$ and $V = \emptyset$
\WHILE{$e^p-e^i>\epsilon$}
\FOR{$u$ in $U$}
\STATE Train $f_u^i$ with loss $e^i_u$ using $U \setminus u$ and $V \cup u$ 
\ENDFOR
\STATE $m = \text{argmin}_{u: u \in U} e^i_u$
\STATE $U = U \setminus m $, $V = V \cup m$
\STATE $e^i = e^i_m$, $f^i = f^i_m$
\STATE $i = i + 1$
\ENDWHILE
\STATE \textbf{Output}: $f^{i-1}$
\end{algorithmic}
\end{algorithm}

\subsection{Generalized groves of neural additive model}

We present generalized groves of neural additive models (GGNAMs), which further sparsify nonlinear interactions in \eqref{eq:GNML}. Features are divided into two categories:
\begin{itemize}
    \item Linear component $\*x_U$, similar to \eqref{eq:GNML}.
    \item Nonlinear component $\*x_V$: subsets $v \in V$ are allowed to be either a single element or multiple elements that allow interactions.
\end{itemize}
Consequently, there are three types of features: linear features, individual nonlinear features, and interacted nonlinear features. It is worth emphasizing that GGNAMs allow different interactions groups to avoid global interactions. Different from \eqref{eq:GAM} and \eqref{eq:GNML}, GGNAMs have the following form:
\begin{align} \label{eq:GSNAM}
    g(\mathbb{E}[y|\*x]) = \alpha + \sum_{u: u \in U} \beta_{u} x_{u} + \sum_{v: v \in V} f_v(\*x_v).
\end{align}
As an simple example, we can have
\begin{align*}
    g(\mathbb{E}[y|\*x]) = \alpha + \beta x_1 + f_1(x_2) + f_2(x_3,x_4),
\end{align*}
which is visualized in Figure~\ref{fig:GSNAM}. GGNAMs offer advantages over NAMs by allowing for linear components and complicated interactions.

\begin{figure}
  \centering
  \includegraphics[scale=0.4]{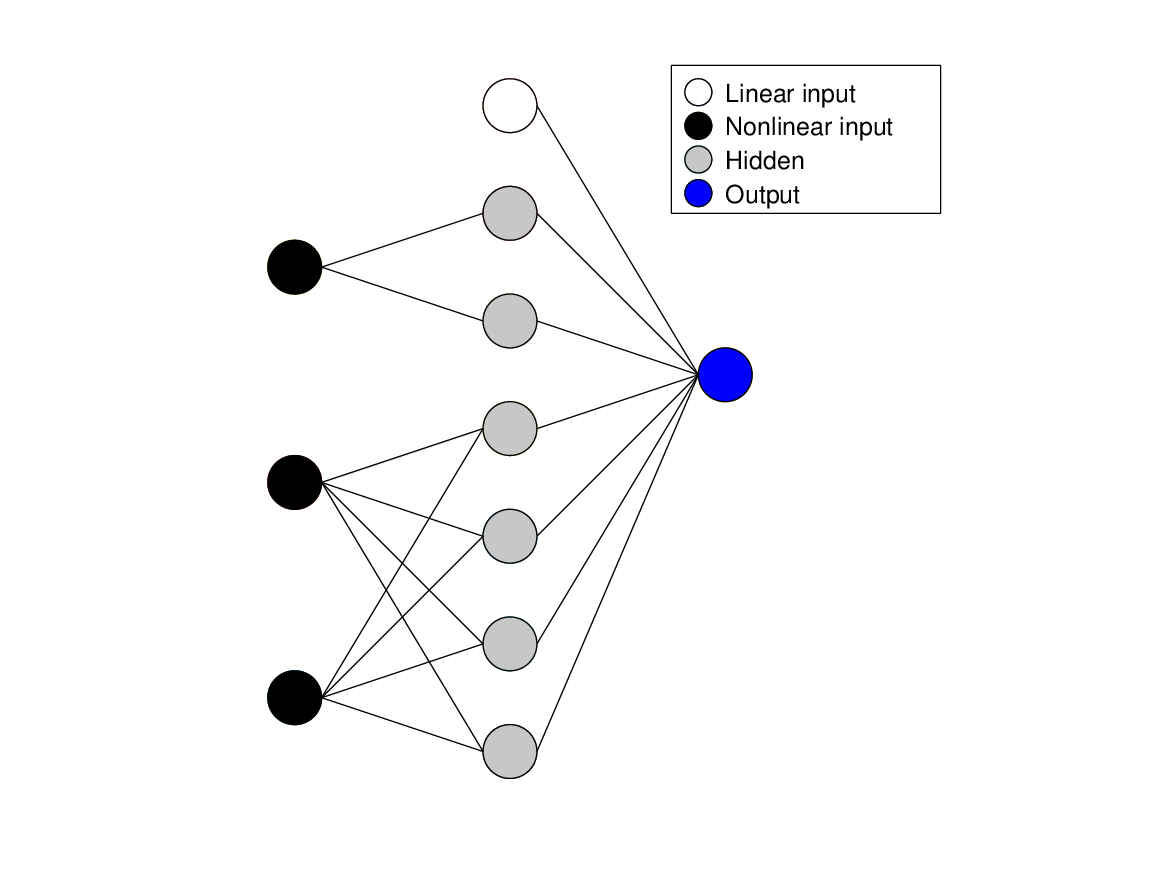}
  \caption{An architecture of the GGNAM. It splits features into linear component, individual nonlinear component, and interacted nonlinear component. }
  \label{fig:GSNAM}
\end{figure}

\subsection{Identify statistical interactions}

While GGNAMs have enabled flexible architectures that can take into account a variety of possibilities, we still require a convenient method for determining the architecture of GGNAMs. As a first step, we employ the forward selection algorithm~\ref{alg:forward_selection} to reduce the dimensions of nonlinear features. Once linear features are determined, we can now investigate interactions among nonlinear features.

Statistical interactions have been studied extensive in the exist literature \cite{sorokina2008detecting,tsang2018neural,tsang2020does}. For simplicity, suppose there are two groups: $\*x$ can be split into two components $\*x_U$ and $\*x_V$, with $U \cup V = D$ and $U \cap V = \emptyset$, where $D = \{1, \dots, p\}$. Extending it to multiple groups is straightforward. Our goal is to test whether a function is additively separable. 
\begin{definition}
We say a function $f$ with $D$ is strictly additive separable for $U$ and $V$ if 
\begin{align}
f(\*x)  = g(\*x_U) + h(\*x_V)
\end{align}
for some functions $g$ and $h$, $U \cup V = D$, and $U \cap V = \emptyset$. 
\end{definition}
We wish to identify these two groups. Both groups have a variety of features and exhibit complex interrelationships, which presents a challenge. To simplify the procedure, we focus on pairwise features $x_i$ and $x_j$ and examine whether there are direct interactions between them. That is,  we wish to know if $x_i \in U$ and $x_j \in V$, where $U \cap V = \emptyset$. Unfortunately, $U$ and $V$ are unknown in advance. To avoid explicitly requiring information for $U$ and $V$, we generalize the idea to allow duplicates for both $U$ and $V$.
\begin{definition}
We say a function $f$ with $D$ is additive separable for $U$ and $V$ if 
\begin{align}
f(\*x)  = g(\*x_U) + h(\*x_V)
\end{align}
for some functions $g$ and $h$, $U \cup V = D$, and $U$ and $V$ are not necessarily disjoint. 
\end{definition}
Clearly, additive separability covers strictly additive separability. Without knowing $U$ and $V$, this generalization allows us to reconsider larger sets $D \setminus i$ and $D \setminus j$, which cover $U$ and $V$. In addition, the following theorem can be easily verified. 
\begin{theorem}[Additive Separability Theorem] \label{thm:separable}
Suppose $f$ with $D$ can be split into two disjoint sets $U, V$ with $i \in U$ and $j \in V$, then $f$ is additive separable for $D \setminus i$ and $D \setminus j$ if and only if $f$ is strictly additive separable for $U$ and $V$. 
\end{theorem}
\begin{proof}
If $f$ is separable for $D \setminus i$ and $D \setminus j$, then exists functions $g$ and $h$  such that
\begin{align*}
f(\*x) = g(\*x_{D \setminus i}) + h(\*x_{D \setminus j}).
\end{align*}
But then $g$ is separable for $U \setminus i$ and $V$, and $h$ is separable for $U$ and $V \setminus j$. Therefore, we can write
\begin{align*}
f(\*x) &= g_1(\*x_{U \setminus i}) + g_2(\*x_{V}) + h_1(\*x_U) + h_2(\*x_{V \setminus j}) \\
&= p(\*x_U) + q(\*x_V),
\end{align*}
for some functions $g_1,g_2,h_1,h_2,p,q$. Conversely, we have
\begin{align*}
f(\*x) &= p(\*x_U) + q(\*x_V) = p_1(\*x_{U \cup V \setminus j}) + q_1(\*x_{U \setminus i \cup V}) \\
&= q_1(\*x_{D \setminus i}) + p_1(\*x_{D \setminus j})
\end{align*}
for some functions $p_1$ and $q_1$. 
\end{proof}
We learn from Additive Separability Theorem~\ref{thm:separable} that in order to check the separability of $x_i$ and $x_j$, we do not necessarily need to consider specific $U$ and $V$. Therefore, we follow \cite{sorokina2008detecting}'s approach and test directly between $\*x_{D \setminus i}$ and $\*x_{D \setminus j}$. NNs have the universal approximation theorem, which allows them to learn $f$ arbitrarily well under a general assumption, as well as $g$ and $h$ in the additive form $f=g+h$. It is, therefore, convenient to compare a FCNN to a GGNAM with two nonlinear groups $D \setminus i$ and $D \setminus j$. We did not allow duplicate features in subsets in our definition of \eqref{eq:GSNAM} for better demonstration, however, we should note there is no problem with allowing duplicate features. Then, $x_i$ and $x_j$ can be separated if there is almost no difference in validation process, for example, 
\begin{align}
\mathbb{E}[|f(\*x)-g(\*x_{D \setminus i}) - h(\*x_{D \setminus j})|] < \epsilon.
\end{align}
Note that this is only one example and other metrics are certainly acceptable. Motivated by this, suppose the model accuracy is denoted by acc($f$), we consider a separability matrix $\*A$ s.t. 
\begin{align} \label{eq:sep_matrix}
A_{i,j} =  
\begin{cases}
\text{acc}(f(\*x)), & \ \text{if } i = j, \\
\text{acc}(g(\*x_{D \setminus i}) + h(\*x_{D\setminus j})), & \ \text{if } i \neq j.
\end{cases} 
\end{align}
If $|A_{i,i}-A_{i,j}|<\epsilon$ for some small $\epsilon$, then we conclude that there is no interactions. $\*A$ offers an intuitive understanding of the consequences arising from decoupling features. In selecting $\epsilon$, there is a trade-off between accuracy and interpretation, which should be determined by problems and users' appetites. Different $\epsilon$ can be useful even for the same dataset for various purposes. When predicting housing prices, a private company might choose a very small $\epsilon$ to seek arbitrage opportunities, whereas a finance researcher might choose a large $\epsilon$ to improve explanation when studying market efficiency. After interactions have been identified, features are grouped together if they have direct interactions or indirect interactions through some intermediate features. The overall procedure can be outlined in Algorithm~\ref{alg:GSNAM}.

\begin{algorithm} [ht!]
\caption{Train a GGNAM} 
\label{alg:GSNAM}
\begin{algorithmic}[1]
\STATE \textbf{Inputs}: $\epsilon>0$, $D=\{1, \dots, p\}$
\STATE Train $f_{\text{LaLR}}$ by a LaLR
\STATE Train $f_{\text{FCNN}}$ by a FCNN
\STATE Perform forward selection algorithm~\ref{alg:forward_selection} with $f_{\text{LaLR}}$ and $f_{\text{FCNN}}$ to split $D$ into linear component $U$ and nonlinear component $V$
\STATE Calculate the separability matrix $\*A$ by \eqref{eq:sep_matrix}
\STATE Split $V$ into disjoint $v_i$ based on $\*A$
\STATE Train a $f_{\text{GGNAM}}$ with $U$ and $V$
\STATE \textbf{Output}: $f_{\text{GGNAM}}$
\end{algorithmic}
\end{algorithm}


\section{Empirical examples}

This section evaluates the performance of models for two classification problems and one regression problem. Specifically, we compare linear and logistic regressions (LaLRs), fully-connected neural networks (FCNNs), neural additive models (NAMs), and generalized groves of neural additive models (GGNAMs).  We randomly split datasets into training ($80\%$) and test ($20\%$). For hyperparameter tuning of separating linear and nonlinear features as well as interactions, a 20$\%$ hold-out validation set is further split from the training set. The area under the curve (AUC) for the test set is used to measure performance in classification and the root mean-square-error (RMSE) is used in the regression. Then, we provide a detailed post-analysis. For example, function behaviors associated with certain features are plotted and analyzed. In such cases, intercepts are subtracted as they are not relevant. 

For model architectures, we use the same structure for FCNNs and individual NNs in NAMs and GGNAMs so that they have the same number of parameters.  For classification, NNs contain one hidden layer with five units, logistic activation, and no regularization. For regression, NNs contain two hidden layers with $[16,8]$ units, ReLU activation, and a $L^2$ penalty is applied with $\lambda = 2e^{-4}$. Since such simple architectures are compared favorably with results from the literature, we do not explore more complex ones. 



Accuracy results are summarized in Table~\ref{tab:all_result} and architectures of GGNAMs are summarized in Table~\ref{tab:archi_result}. According to these results, ML models outperform the LaLR, indicating the power of ML models. 
GGNAMs achieve the same level of accuracy as even the most complex FCNNs, indicating that they are capable of approximating complex datasets. In addition, GGNAMs have very transparent architectures. There are a large number of linear components and a very small number of interacted nonlinear components. These architectures are much simpler than FCNNs, as well as NAMs in many instances. Next, we will provide more details and interpretations of empirical examples. 


\begin{table}[h]
    \centering
    \caption{Model performance for all datasets. GGNAMs achieve the same level of accuracy to FCNNs.}
    \label{tab:all_result}
    \begin{tabular}{ccccc}
    \hline
    Methods/Dataset  & TCS (AUC) & PCB (AUC) & GE (RMSE)  \\ \hline
    LaLR  & $0.721$ & 0.677 & 0.151  \\ \hline
    FCNN & $0.759$ & 0.908 & 0.142 \\ \hline
    NAM  & $0.759$ & 0.885 & 0.140 \\ \hline
    GGNAM & $0.766$ & 0.907 & 0.141\\ \hline 
    \\
    \end{tabular}
\end{table}

\begin{table}[h]
    \centering
    \caption{Architectures of GGNAMs. 
    The dominant components are linear and the interacted nonlinear components are sparse.
    }
    \label{tab:archi_result}
    \begin{tabular}{ccccc}
    \hline
    Dataset/Arch  & Linear & Indi Nonlinear & Nonlinear Groups  \\ \hline
    TCS & $x_2-x_{5}$ & $x_{1},x_6$ &   \\ 
        & $x_7-x_{23}$ & & \\ \hline
    PCB & $x_1-x_{16}$ & $x_{19}$ & $(x_{17},x_{22})$ \\ 
        & $x_{18}, x_{20}, x_{21}$ & & \\ 
        & $x_{23} - x_{34}$ & & \\ \hline
    GE & $x_1-x_8$ & $x_9$ & & \\ 
       & $x_{10} - x_{13}$ & &  \\ \hline
    \end{tabular}
\end{table}

\subsection{Taiwan credit scoring data}

\subsubsection{Data description}

Taiwan credit scoring (TCS) dataset \cite{yeh2009comparisons} is concerned with clients' probability of default (PoD). Card-issuing banks over issued cash and credit cards to unqualified applicants, which led to high delinquencies for banks during this period. The study aimed to identify high risk clients based on their credit histories and to deny applications that were deemed unlikely to be repaid. In credit scoring, interpretation by regulators is strictly mandatory \cite{CFPB2022credit}: a detailed explanation must be given to customers upon denial. Please see \cite{arrieta2020explainable} for a more detailed explanation of the requirements.

\subsubsection{Results}


By forward selection algorithm, out of 23 features, only two are selected as nonlinear features, namely $x_6$ and $x_{1}$ in order, where $x_1$ calculates the amount of the given credit and $x_6$ is the repayment status in September 2005. This significantly reduces the dimensionality of nonlinear features. Then, we calculate the separability matrix \eqref{eq:sep_matrix} based on AUC in Table~\ref{tab:taiwan_inter}. Separating these features is not harmful. Accordingly, we confirm that there are no interactions. We then visualize selected nonlinear features in Figure~\ref{fig:taiwan_6_1}. As the result of $x_1$ illustrates, clients with more given credits are likely to be more reliable, as banks would not extend credits if they did not trust them. The monotonic behavior of $x_6$ starting from 0, illustrates that with delayed payments, clients are more likely to default. Additionally, for these two features, nonlinearities can be summarized by diminishing marginal effects (DMEs). This means that additional changes will have a diminishing effect. As an example, the first delinquency will increase the client's risk level significantly, while delinquencies of four or five times do not seem to matter very much. DMEs are commonly found in social science, and nonlinearity of this kind leads to the failure of the LaLR.  Using a LaLR would easily underestimate the risk of past-due payments when there are few, and overestimate the risk when there are a lot. The inability to capture DMEs, in this case, can be presented as the primary reason for the use of GGNAMs over LaLRs to regulators.  In theory, we would expect DMEs to also apply for other features, including those related to delinquency. Nevertheless, there is insufficient evidence to determine their significance. Non-linearities that are too weak to be detected are ignored in favor of a simpler model of transparency. Such a visualization is easy by explainable ML methods, such as NAMs and GGNAMs, but is inapplicable to more complex ML methods, such as FCNNs. The GGNAM achieved the simplest architecture with an equal degree of accuracy in this dataset, therefore it should be preferred.

\begin{table}[h]
    \centering
    \caption{Separability matrix of the GGNAM for the TCS dataset. }
    \label{tab:taiwan_inter}
    \begin{tabular}{cccccc}
    \hline
    Feature   & $x_{1}$  & $x_{6}$ & \\ \hline
    $x_{1}$ & $0.764 $ & $0.761$ &  \\ \hline
    $x_{6}$ &          & $0.764$ & \\ \hline
    \\
    \end{tabular}
\end{table}



\begin{figure}[h]
    \centering
    \caption{Output of the GGNAM with respect to $x_1$ and $x_6$ in the TCS dataset. Diminishing marginal effects are observed for both features. }
    \label{fig:taiwan_6_1}
    \includegraphics[scale=0.4]{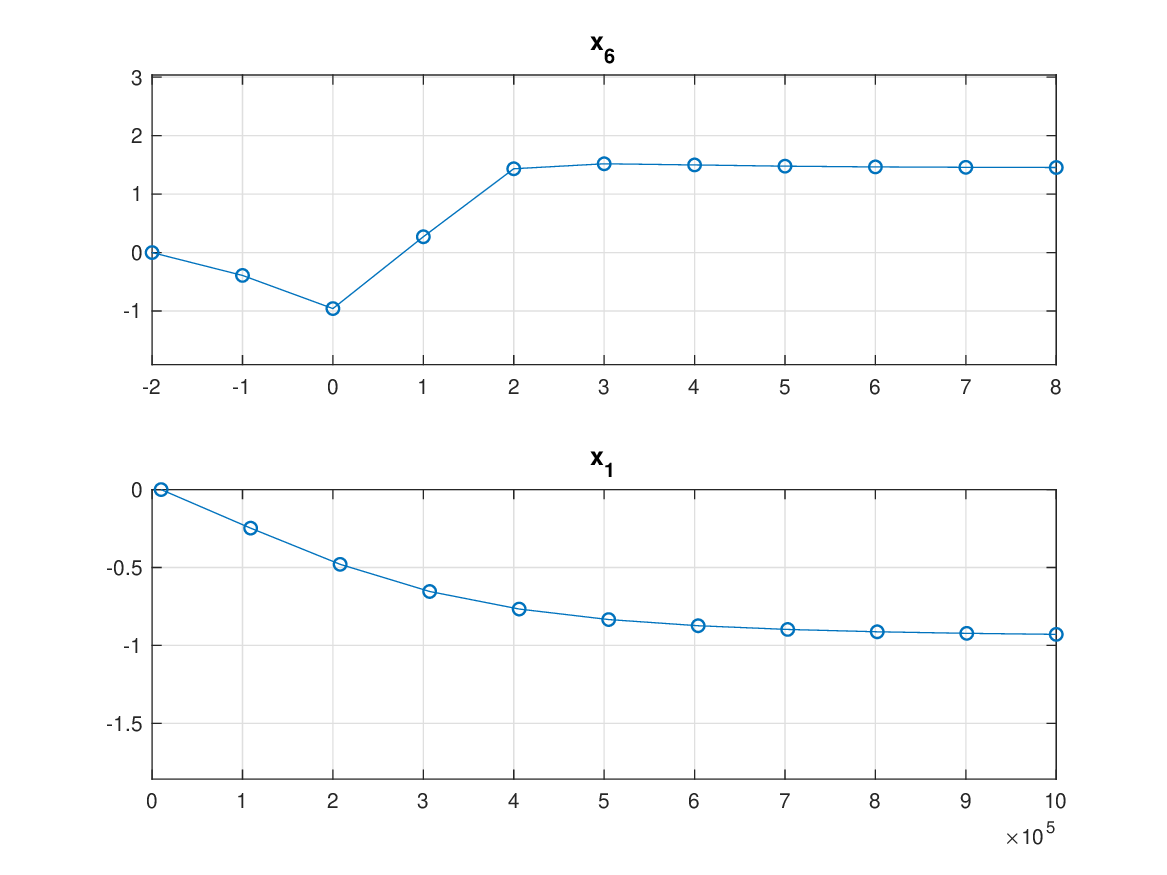}
\end{figure}

\subsection{Polish companies bankruptcy dataset}

\subsubsection{Data description}

The dataset presented here represents predictions regarding bankruptcies of Polish companies \cite{zikeba2016ensemble}. There has always been an active field in business and economics dedicated to bankruptcy prediction \cite{li2010hybrid}, and its significance cannot be emphasized enough. The predictive power of such studies is important, but from a financial perspective, it seems more important to understand their reasoning. The study of bankruptcy causes could potentially provide a wealth of information to regulators and policymakers. 

Bankrupt companies were analyzed from 2000 to 2012, while still operating companies were evaluated from 2007 to 2013. 
It contains 7027 samples, of which 271 represent bankrupt companies. 64 indicators are used to represent the business conditions of companies, which is a common approach used in accounting research. It is noticeable that some indicators are highly correlated or contain overlapping information. In order to de-noise the data, we eliminated highly correlated ($>$0.95 or $<$-0.95) features and left with 34 features.  We believe that this is a necessary practice for this study, as highly correlated features complicate interpretations of their effects. In addition, there are 5835 missing entries, which have been replaced by averaged feature values.

\subsubsection{Results}


Surprisingly, by forward stepwise selection, out of 34 features, only three features, $x_{17}$, $x_{19}$, and $x_{22}$, are identified as nonlinear features, leading to a much smaller dimension which we didn't foresee. In this dataset, $x_{17}$ calculates $\frac{\text{profit on operating activities}}{\text{financial expenses}}$, $x_{19}$ records the logarithm of total assets, and $x_{22}$ calculates $\frac{\text{operating expenses}}{\text{total liabilities}}$. The next step is to determine whether certain nonlinear features can be separated. According to the separability matrix in Table~\ref{tab:bankrupt_inter}, we conclude that only $x_{17}$ and $x_{22}$ definitely cannot be separated. We find this rather intriguing. The interaction among these features is to be expected; however, we did not anticipate an architecture that would be so sparse. The significance of these two features is not immediately apparent, among others. Nevertheless, if interactions are not necessary, then they should be omitted for transparency.
On the one hand, this result supports the conventional linear approach, such as Altman's Z-score \cite{altman1968financial}, that linear models are sufficient to provide fairly accurate results. For this reason, linear relationships are favored to provide greater transparency and explainability. On the other hand, the results indicate that ML models could improve the model's performance. Certain features exhibit strong nonlinearity and cannot be well approximated by simple linear models. A GGNAM model has been compromised in both ways, similar to LaLR but with comparable accuracy to FCNN. 


To verify whether such an interaction indeed exists, we calculate their joint marginal probability of bankruptcy (JMPoB) based on features $x_{17}$ and $x_{22}$ and record these results in a matrix, visualized in  Figure~\ref{fig:bank_17_22}. Both features are evenly divided into ten intervals, with approximately 700+ samples in each interval. Due to uneven distributions of features, we did not specify axes to facilitate better visualization. There are 100 boxes in total for two features with ten intervals each. In the absence of sufficient samples ($<30$) in the box, we denote the probability as $-0.1$, for purposes of visualization. It is evident that an abnormal pattern has been observed: for samples taken in the $10^{\text{th}}$ interval of $x_{17}$: the JMPoBs for small amounts of $x_{22}$ are exceptionally high. Patterns such as these are only found for large $x_{17}$, indicating that these features interact. The GGNAM allows us to demystify the black-box structure. In the following example, we calculate $f(x_{17},x_{22})$ at median values of samples within intervals of $x_{17}$ and $x_{22}$. Figure~\ref{fig:bank_GSNAM_17_22} demonstrates that the characteristics described above are somewhat reflected in this figure. The observed patterns suggest that in normal circumstances, a larger percentage of operating expenses over total liabilities increases the risk of bankruptcy. However, when a company makes significant profits from its operations over its financial expenses, a small percentage of operating expenses over total liabilities may be indicative of bankruptcy. In such a case, we are reminded to focus not only on the features' main effects but also on their interaction effects. It also suggests that in extreme circumstances, abnormal patterns might occur in financial datasets and we may not be able to easily extend the linear relationships and common knowledge to such circumstances. Complex ML models offer a solution to modeling such complex phenomena. This type of interaction has not yet been included in LaLRs and NAMs \cite{agarwal2021neural}. 
Visualizations like these in GGNAMs should help us better comprehend the dataset and figure out the cause of the bankruptcy.

\begin{table}[h]
    \centering
    \caption{Separability matrix of the GGNAM for the bankruptcy dataset}
    \label{tab:bankrupt_inter}
    \begin{tabular}{ccccccc}
    \hline
    Feature & $x_{17}$ & $x_{19}$ & $x_{22}$ \\ \hline
    $x_{17}$ & 0.913 & 0.908    & 0.881    \\ \hline
    $x_{19}$ &       & 0.913    & 0.910      \\ \hline
    $x_{22}$ &       &          & 0.913   \\ \hline
    \\
    \end{tabular}
\end{table}

\begin{figure}[h]
    \centering
    \includegraphics[scale=0.4]{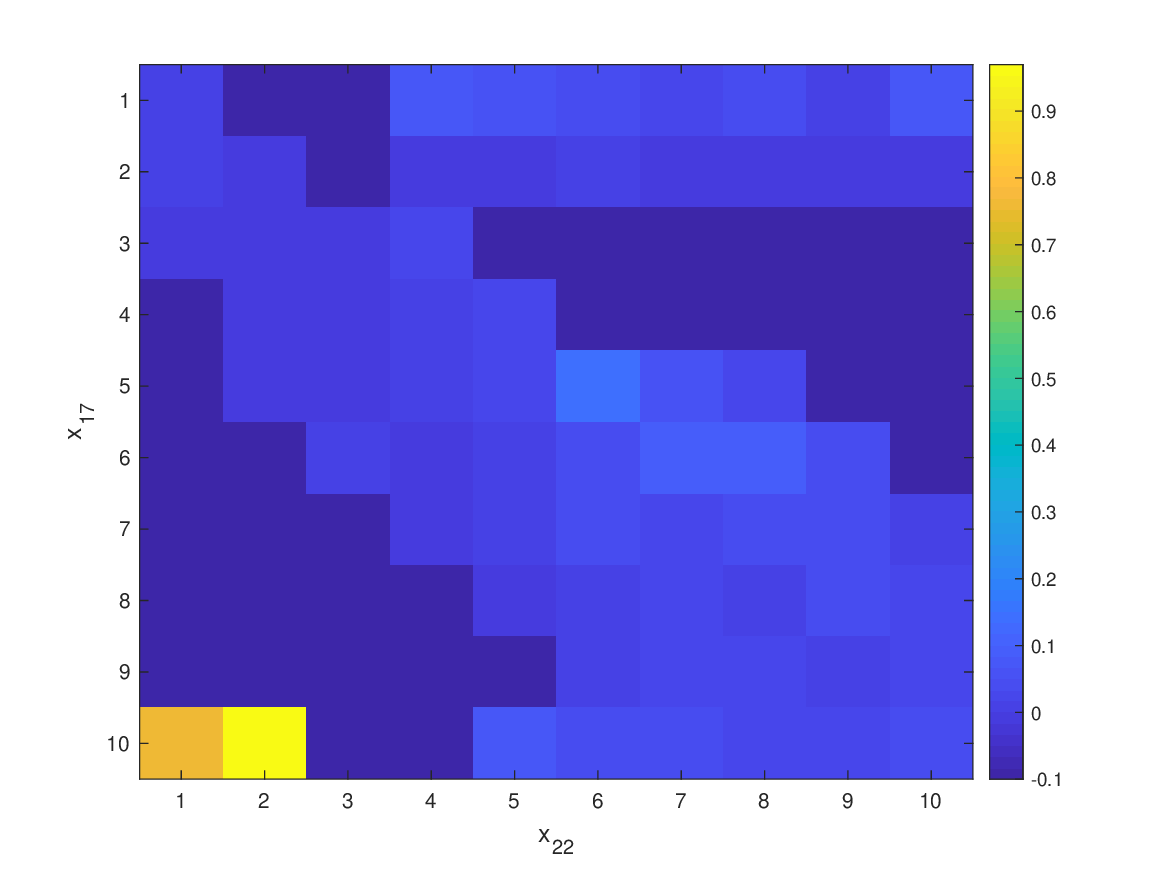}
    \caption{Joint marginal probabilities of bankruptcy for the features $x_{17}$ and $x_{22}$ for the bankruptcy dataset. In the $10^{\text{th}}$ interval of $x_{17}$, probabilities are exceptionally high in $1^{\text{th}}$ and $2^{\text{th}}$ interval of $x_{22}$. }
    \label{fig:bank_17_22}
\end{figure}



\begin{figure}[h]
    \centering
    \includegraphics[scale=0.4]{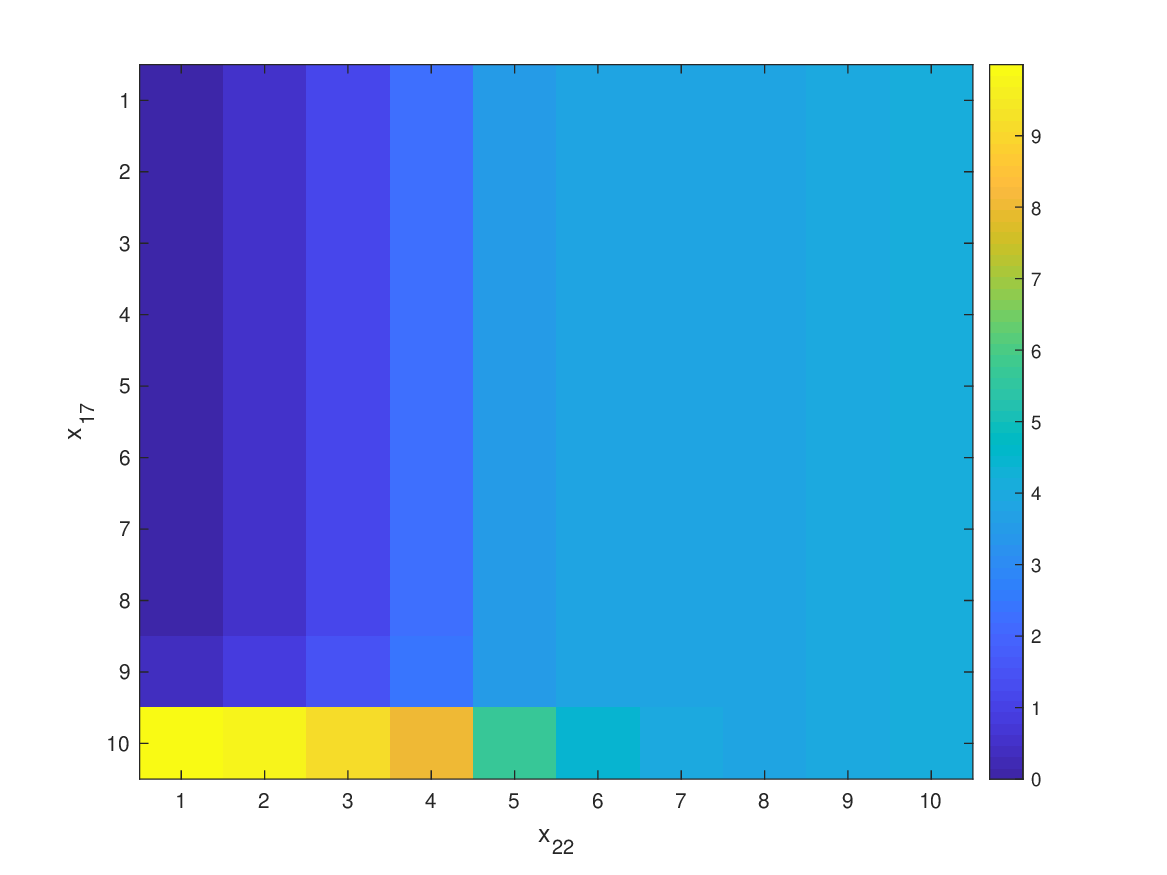}
    \caption{Output of the GGNAM with respect to $x_{17}$ and $x_{22}$ in the bankrupt dataset. The output is monotonically increasing with respect to $x_{22}$ in $1^{\text{th}}-9^{\text{th}}$ intervals of $x_{17}$, but is monotonically decreasing in $10^{\text{th}}$ interval of $x_{17}$. }
    \label{fig:bank_GSNAM_17_22}
\end{figure}

\subsection{Productivity prediction of garment employees dataset}

\subsubsection{Data description} This dataset aims to predict the productivity of garment workers \cite{al2019deep,imran2021mining}. The specific objective of the study was to determine if there was a discrepancy between the targeted productivity set by authorities and the actual productivity in order to minimize potential losses. In spite of the objective to focus on the predictive power, it is important to keep in mind the possible ethnic and legal implications. Consider the following example: If a ML method predicts working overtime will enhance productivity, it would be unwise to recommend authorities increasing overtime. To better understand features with sensitive information, such as overtime, a transparent model would be desirable. 

Detailed data were collected from the industrial engineering department of a garment manufacturing facility of a reputed company in Bangladesh. This dataset contains the production data for the sewing and finishing department for three months between January 2015 and March 2015. The dataset consists of 1197 samples and includes 14 features. We neglect the date feature and summarize the rest of features as follows:
There are 506 missing entries for $x_6$, and we replace them with the average value as in the last example.

\subsubsection{Results}
Based on the forward selection algorithm, only $x_{9}$ is identified as the nonlinear feature, so there's no need to check interactions. $x_9$ represents the amount of financial incentive that enables or motivates a particular course of action. Figure ~\ref{fig:garments_GSNAM_9} clearly demonstrates the DME of financial incentive. On the basis of the above observation, one might suggest providing financial incentives, but only within certain limits. The GGNAM simplifies the NAM further by only using one nonlinear feature.


\begin{figure}
    \centering
    \includegraphics[scale=0.4]{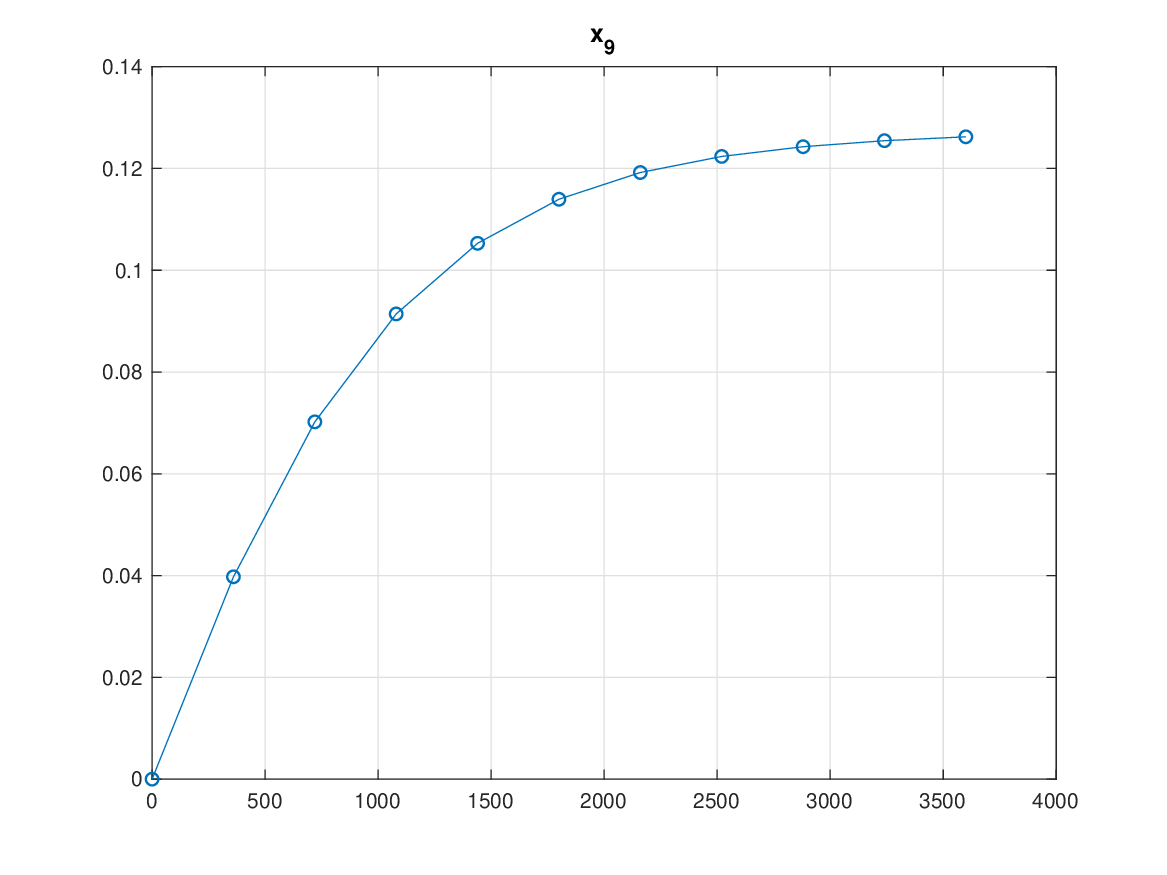}
    \caption{Output of the GGNAM with respect to $x_{9}$ in the garment dataset. The diminishing marginal effect is observed.}
    \label{fig:garments_GSNAM_9}
\end{figure}

\section{Future Work}


Currently, the existing approach does not take domain knowledge into account
. Recent research has demonstrated that domain-knowledge-inspired ML models \cite{chen2023address,chen2022monotonic,gupta2020multidimensional,repetto2022multicriteria} could lead to more conceptually sound and fair models. According to \cite{chen2023address}, detecting statistical interactions without considering domain knowledge can result in oversimplified and unreasonable models. We will explore how to build transparent models with domain knowledge in future work. 


\bibliographystyle{IEEEtran}
\bibliography{chen}

\end{document}